\def\year{2019}\relax
\documentclass[letterpaper]{article} 
\usepackage{aaai19}  
\usepackage{times}  
\usepackage{helvet}  
\usepackage{courier}  
\usepackage{url}  
\usepackage{graphicx}  
\usepackage{xspace}
\usepackage{algpseudocode}
\usepackage{algorithm}
\usepackage{amssymb}
\usepackage{amsmath}
\usepackage{amsthm}
\usepackage{hhline}
\usepackage{soul} 
\usepackage[normalem]{ulem} 
\usepackage{xcolor}
\usepackage{multirow}
\usepackage{array}
\usepackage{caption}
\usepackage{tikz}
\usepackage{graphicx}
\usepackage{pgfplots}
\usepackage{wasysym}
\usepackage{ctable} 
\usepackage{etex}
\usepackage{enumitem}
\usepackage{url}
\usepackage{balance}

\newcolumntype{L}[1]{>{\raggedright\arraybackslash}p{#1}}
\newcolumntype{C}[1]{>{\centering\arraybackslash}p{#1}}
\newcolumntype{R}[1]{>{\raggedleft\arraybackslash}p{#1}}

\newtheorem{proposition}{Proposition}
\newtheorem{theorem}{Theorem}
\newtheorem{corollary}{Corollary}
\newtheorem{definition}{Definition}

\newcommand{\setcover}{\textsc{Set Cover}\xspace}
\newcommand{\exactsetcover}{\textsc{Exact Set Cover}\xspace}
\newcommand{\redblue}{\textsc{Red-Blue Set Cover}\xspace}

\newcommand{\eval}[2]{{Eval}$({#1},#2)$\xspace}
\newcommand{\error}[3]{\mbox{\textsc{Error}}_{\text{#1}}({#2},{#3})}
\newcommand{\psc}{$\pm$PSC\xspace}

\newcommand{\ruleselectgeneral}[1]{\textsc{Rule-Select}$_{\text{#1}}$\xspace}
\newcommand{\ruleselect}[1]{{\textsc{Rule-Select}}$_{\text{#1}}$\textup{(a,r)}\xspace}

\newcommand{\exactruleselectgeneral}[1]{\textsc{Exact Rule-Select}$_{\text{#1}}$\xspace}
\newcommand{\exactruleselect}[1]{\textsc{Exact Rule-Select}$_{\text{#1}}$\textup{(a,r)}\xspace}

\newcommand{\optruleselectgeneral}[1]{\textsc{Min Rule-Select}$_{\text{#1}}$\xspace}
\newcommand{\optruleselect}[1]{\textsc{Min Rule-Select}$_{\text{#1}}$\textup{(a,r)}\xspace}

\newcommand{\pscproblem}{\textsc{Positive-Negative Partial Set Cover\xspace}}

\newcommand{\paretooptimality}[1]{\textsc{Pareto Opt Solution Rule-Select}$_{\text{#1}}$\textup{(a,r)}\xspace}

\newcommand{\paretofrontmembership}[1]{\textsc{Pareto Front Membership Rule-Select$_{\text{#1}}$}\textup{(a,r)}\xspace}

\newcommand{\setcoveroptimality}{\textsc{Set-Cover Optimality}\xspace}

\newcommand{\bileveloptimalsolution}[1]{\textsc{Bi-level Opt Solution Rule-Select}$_{\text{#1}}$\textup{(a,r)}\xspace}

\newcommand{\bileveloptimalvalue}[1]{\textsc{Bi-level Opt Value Rule-Select}$_{\text{#1}}$\textup{(a,r)}\xspace}

\newcommand{\eat}[1]{{}}

\newcommand{\SortNoop}[1]{}

\newcommand{\squishlist}{
\begin{list}{$\bullet$}
{
\setlength{\itemsep}{0pt}
\setlength{\parsep}{3pt}
\setlength{\topsep}{3pt}
\setlength{\partopsep}{0pt}
\setlength{\leftmargin}{1.5em}
\setlength{\labelwidth}{1em}
\setlength{\labelsep}{0.5em} } }

\newcommand{\squishlisttwo}{
\begin{list}{$\bullet$}
{
\setlength{\itemsep}{0pt}
\setlength{\parsep}{0pt}
\setlength{\topsep}{0pt}
\setlength{\partopsep}{0pt}
\setlength{\leftmargin}{2em}
\setlength{\labelwidth}{1.5em}
\setlength{\labelsep}{0.5em} } }

\newcommand{\squishend}{
\end{list}  }

\begin{document}
%
\title{Knowledge Refinement via Rule Selection}
\begin{scriptsize}
\author{Phokion G. Kolaitis\textsuperscript{1,2} ~~Lucian Popa\textsuperscript{2} ~~Kun Qian\textsuperscript{2}\\
\textsuperscript{1} UC Santa Cruz ~~ \textsuperscript{2} IBM Research -- Almaden \\
kolaitis@ucsc.edu, lpopa@us.ibm.com, qian.kun@ibm.com
}
\end{scriptsize}

\maketitle
\begin{abstract}
\looseness=-1 In several different applications, including data transformation and entity resolution, rules are used to capture aspects of knowledge about
the application at hand. Often, a large set of such rules is generated automatically or semi-automatically, and the challenge is to
refine the encapsulated knowledge by selecting a subset of rules based on the expected operational behavior of the rules on
available data.
In this paper, we carry out a systematic complexity-theoretic investigation of the following rule selection problem:  given a set of rules
specified by Horn formulas, and a pair of an input database and an output database, find a subset of the rules that
minimizes the total error, that is, the number of false positive and false negative errors arising from the selected rules. We first establish
computational hardness results for the decision problems underlying this minimization problem, as well as upper and lower bounds for its
approximability. We then investigate a bi-objective optimization version of the rule selection problem in which both the total error and the size of the
selected rules are taken into account.  We show that testing for membership in the Pareto front of this bi-objective optimization problem
is DP-complete. Finally, we show that a similar DP-completeness result  holds for a bi-level optimization version of the rule selection problem, where
one minimizes first the total error and then the size.
\end{abstract}


\looseness = -1 Rules, typically expressed as Horn formulas, are ubiquitous in several different areas of computer science and artificial intelligence. For example, rules are the basic construct of (function-free) logic programs. In  data integration \cite{DBLP:conf/pods/Lenzerini02} and data exchange \cite{FAGIN200589},
  rules are known as GAV (global-as-view) constraints and are used to specify data transformations between a local (or source) schema and a global (or target) schema.
  In data mining, rules  have many uses, including the specification of   contextual preferences   \cite{DBLP:conf/sigmod/AgrawalRT06,DBLP:journals/is/AmoDDGLS15}. In entity resolution, rules have been used to specify   blocking functions \cite{DBLP:conf/icdm/BilenkoKM06} and  entity resolution algorithms \cite{Qian2017}.

 Often, a large set of  rules is generated automatically or semi-automatically, and the challenge is to
refine the encapsulated knowledge by selecting a subset of rules based on the expected operational behavior of the rules on
available data. Rule selection  arises naturally in all aforementioned contexts and, in fact, in most contexts involving reasoning about data.
Here, we present an example
motivated by a real-life application in which we are building a knowledge base of experts in the medical domain, based on public data, and where entity resolution is one of the crucial first steps.

In entity resolution, the aim is to identify references of the same real-world entity across multiple records or datasets.
Consider the scenario depicted in
Figure 1, where the aim is to identify occurrences of the same author across
research publications from {\tt \small PubMed}\footnote{https://www.ncbi.nlm.nih.gov/pubmed}.
\begin{figure}[ht]
\hspace*{-.2cm}
\centering
\includegraphics[width=8.8cm]{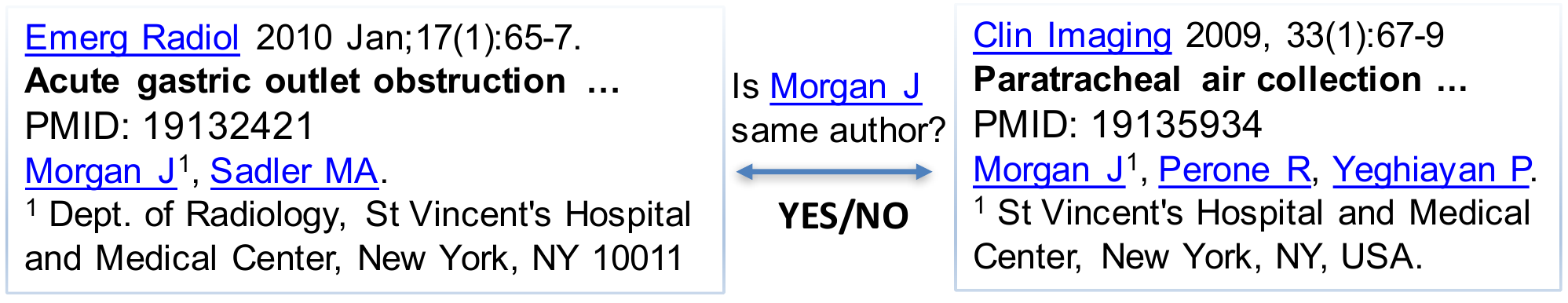}
\label{example}
\vspace*{-.5cm}
\caption{\small An entity resolution task}
\vspace*{-.2cm}
\end{figure}
\begin{figure}[t]
\hspace*{-.3cm}
\centering
\includegraphics[width=9.5cm]{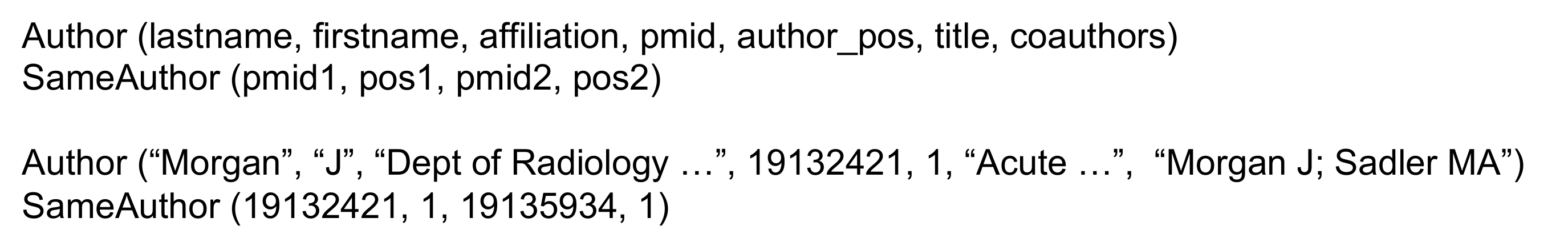}
\label{schemas}
\vspace*{-.5cm}
\caption{\small Schemas and facts}
\vspace*{-.3cm}
\end{figure}
%
%
This entity resolution task can be modeled using a source schema that includes a relation {\tt \small Author} and a link schema that consists of a relation {\tt \small SameAuthor}.  Sample facts (records) over the source and the link relations are given in
Figure 2.
As in the frameworks of Markov Logic Networks~\cite{SinglaDomingos06}, Dedupalog~\cite{Dedupalog09}, and declarative entity linking~\cite{BurdickFKPT16},
explicit link relations are used to represent entity resolution inferences. In particular,
the {\tt \small SameAuthor} fact in Figure 2 represents that an inference was made to establish that the author in position {\tt \small 1} of
publication with {\tt \small pmid 19132421} is the same as the author in position {\tt \small 1} of publication with {\tt \small pmid 19135934}.

\begin{figure}[t]
\hspace*{-.3cm}
\centering
\includegraphics[width=9.2cm]{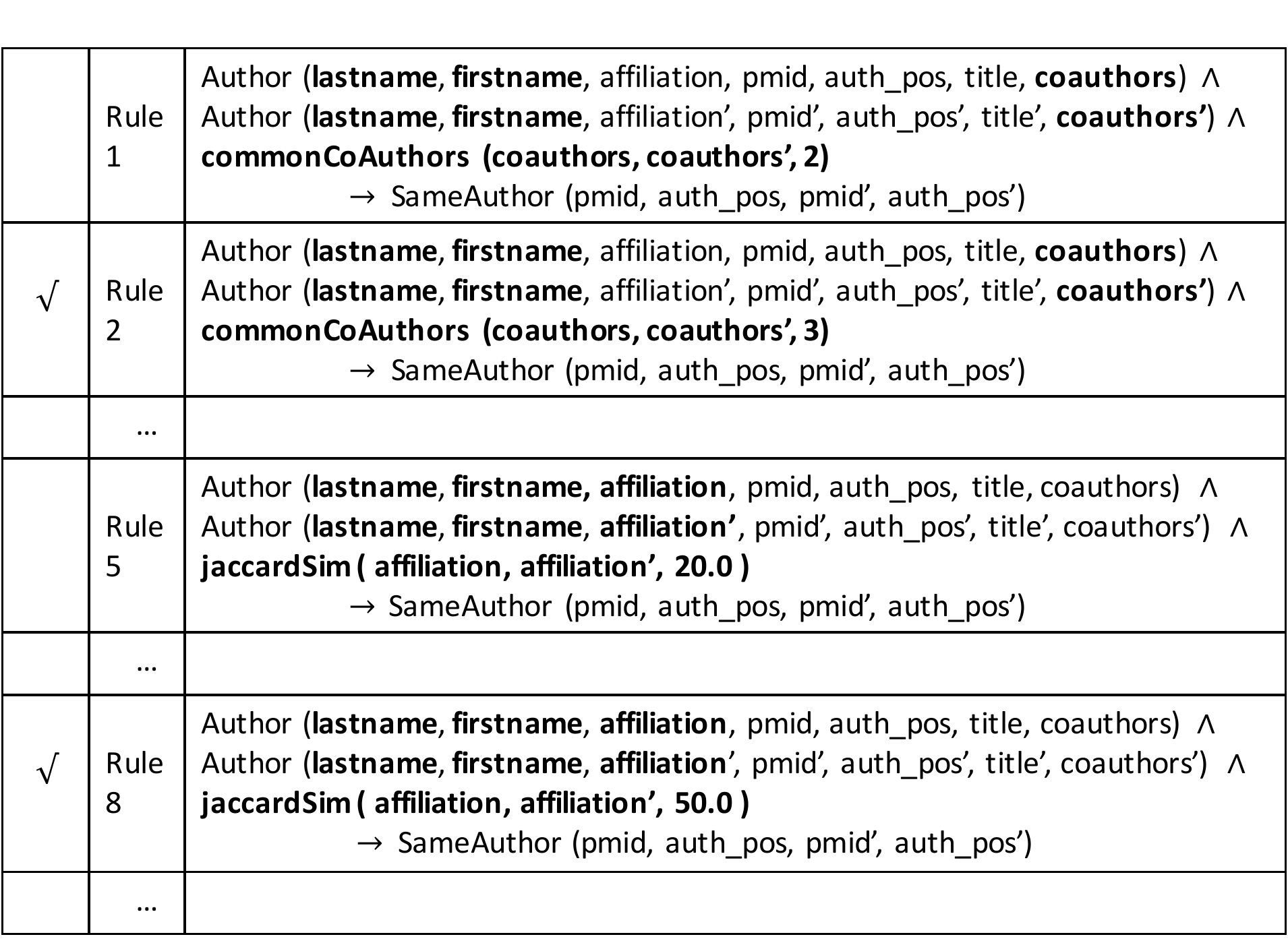}
\label{rules}
\vspace*{-.5cm}
\caption{\small Candidate rules for an entity resolution task}
\vspace*{-.3cm}
\end{figure}

For a given entity resolution task, there is typically a large set of candidate rules that may apply on the input data to form matches among the entities.
For our concrete scenario,
Figure 3
gives a sample of candidate matching rules.
These rules involve the alignment of relevant attributes
(e.g., {\tt \small lastname} with {\tt \small lastname}, {\tt \small affiliation} with {\tt \small affiliation}) and the subsequent application of  similarity predicates, filters, and thresholds.
The challenge is to find a subset of the candidate rules  with the ``right" combinations of predicates
and thresholds that will lead to high precision and recall with respect to a given set of ground truth data. As an example, both {\tt \small Rule~1} and {\tt \small Rule~2} in Figure 3
generate a {\tt \small SameAuthor} link between two author occurrences on two different
publications, provided that the last names and first names are identical and provided that there is a sufficient number of common coauthors on the two publications. However, {\tt \small Rule~1}, which checks for at least two common coauthors, may turn out to be too imprecise (i.e., may yield too many false positives), while {\tt \small Rule~2}, which checks for at least three common coauthors, may result into fewer errors.
 Different rules may use different predicates in their premises. For example, {\tt \small Rules~5-8} exploit the Jaccard similarity of affiliation,
  but have different similarity thresholds.
Only one of them ({\tt \small Rule~8}, with similarity threshold of 50\%) may achieve high enough precision.

 Thus,
 the problem becomes how to select a set of rules that achieve high precision (i.e., minimize the number of false positives) and high recall (i.e.,  minimize the number of false negatives) with  respect to a given set of ground truth data; furthermore, one would also like to select a compact (in terms of size) such set of rules.

\looseness = -1 Similar rule selection problems have been studied in several different contexts.  In data exchange,
\cite{Kimmig17} have investigated the \emph{mapping selection} problem:  given a set $\mathcal C$ of rules expressing data transformations between a source and a target schema, and a pair $(I,J)$ of a source database $I$ and a target database $J$, find a subset $\mathcal C'$ of the rules that minimizes the sum of the false positive errors, the false negative errors, and the sizes of the rules in $\mathcal C'$.  \cite{DBLP:conf/icdt/SarmaPGW10} have investigated the \emph{view selection} problem: given a materialized view $V$, a database $D$, and a collection $\mathcal S$ of sets of rules on $D$, find a set of rules in $\mathcal S$ that is as ``close" to the view $V$ and as compact as possible.
In data mining, \cite{DBLP:conf/sigmod/AgrawalRT06,DBLP:conf/sdm/DavisST09} investigated the problems of selecting contextual preference rules and association rules; these problems can be cast as variants of the rule selection problem considered here.

\looseness = -1 \paragraph{Summary of Results}
We formalize  the rule selection problem   with rules
 specified by Horn formulas of first-order logic;  the relation symbols in the premises of the Horn formulas come from a \emph{premise} schema, while those in the conclusions come from a \emph{conclusion} schema that is disjoint from the premise schema. This formalization captures rule selection problems in a variety of contexts.

\eat{
\looseness = -1 This formalization is flexible enough to capture rules that have been used in several different contexts. For example, in  data integration \cite{DBLP:conf/pods/Lenzerini02} and data exchange \cite{FAGIN200589}, rules are known as GAV (Global-as-View) constraints; furthermore, the premise schema is typically called the \emph{local} or \emph{source} schema, while the conclusion schema is typically called the \emph{global} or the \emph{target} schema. In entity resolution, rules have been used to specify   blocking functions \cite{DBLP:conf/icdm/BilenkoKM06} and  entity resolution algorithms \cite{Qian2017}.
 In data mining, rules can simulate the contextual preference rules studied in \cite{DBLP:conf/sigmod/AgrawalRT06,DBLP:journals/is/AmoDDGLS15}. }

\looseness = -1 An input to the rule selection problem consists of a finite set $\mathcal C$ of  rules and a pair $(I,J)$ of a premise database $I$
and a conclusion database $J$
that represents ground truth. When a subset $\mathcal C'$ of $\mathcal C$ is evaluated on the premise instance $I$, it produces a conclusion instance $\mbox{\eval{\mathcal C'}{I}}$. The set $\mbox{\eval{\mathcal C'}{I}} \setminus J$ is the set of the false positive errors, while the set $J\setminus \mbox{\eval{\mathcal C'}{I}}$ is the set of false negative errors. We  study the optimization problem \optruleselectgeneral{FP+FN} in which, given $\mathcal C$, $(I,J)$ as above,  the goal is to find a subset $\mathcal C'$ of $\mathcal C$ so that the number of false positive and false negative errors is minimized. We also study the optimization problem \optruleselectgeneral{FP} in which the goal is  to find a subset $\mathcal C'$ of $\mathcal C$ so that the number of false positive  errors is minimized and there are no false negative errors (this is meaningful when $J\subseteq \mbox{\eval{\mathcal C}{I}}$).

\looseness = -1 To gauge the difficulty of these two optimization problems, we first examine their underlying decision problems. We show that the decision problems involving a bound on the error are NP-hard. We also show that the \emph{exact} decision problems asking if the error is equal to a given value are DP-hard; in particular, they
 are both NP-hard and coNP-hard (thus, unlikely to be in $\textup{NP} \cup \textup{coNP}$).
In view of these hardness results, we focus on the approximation properties of the two rule selection optimization problems. We show that, in a precise sense, \optruleselectgeneral{FP} has the same approximation properties as
 the \redblue~problem, while \optruleselectgeneral{FP+FN} has the same approximation properties as the \pscproblem~problem. These results yield both polynomial-time approximation algorithms and lower bounds for the approximability of our problems.

\looseness = -1  The preceding results focus on the minimization of the error produced by the selected rules. What if one  wants to also take the size of the selected rules into account? Since error and size are qualitatively incomparable quantities, it is not meaningful to add them or even to take a linear combination of the two. Instead, we consider pairs of values $(e,s)$ of error and size that are \emph{Pareto optimal}, that is, neither of these values can be decreased without increasing the other value at the same time. The \emph{Pareto front} of an instance is the set of all Pareto optimal pairs. Even though the study of Pareto optimality has been a central theme of multi-objective optimization for decades, it appears that no such study has been carried out for rule selection problems in any of the contexts discussed earlier. Here, we initiate such a study and show that the following problem is DP-hard: given a set $\cal C$ of rules, a pair $(I,J)$ of a premise database and a conclusion database, and a pair $(e,s)$ of integers, does $(e,s)$ belong to the Pareto front of \optruleselectgeneral{FP+FN}? We also show that a similar DP-hardness result holds for
\optruleselectgeneral{FP}.

Finally, we investigate a bi-level optimization version of \optruleselectgeneral{FP+FN}, where
one minimizes first the total error and then the size. We show that the following problem is DP-hard:
given a set $\cal C$ of rules, a pair $(I,J)$ of a premise database and a conclusion database, and a pair $(e,s)$ of integers,
is $e$ the minimum possible error and is $s$ the minimum size of subsets of rules having the minimum error? We also show a
similar DP-hardness result holds for the bi-level optimization version of \optruleselectgeneral{FP}.

The main results of this paper are summarized in Table \ref{tbl:summary}, which can be found in a subsequent section.

\section{Related Work}
 We already mentioned that \cite{Kimmig17}  studied the \emph{mapping selection} problem in the context of data exchange. In addition to considering rules specified by Horn formulas (GAV constraints in data exchange), they also considered richer rules in which the conclusion involves existential quantification over a conjunction of atoms (GLAV - global and local as view - constraints in data exchange). They established that the mapping selection problem is NP-hard even for GAV constraints, but did not explore approximation algorithms for this optimization problem; instead, they designed an algorithm that  uses  probabilistic soft logic \cite{DBLP:journals/jmlr/BachBHG17} to solve a relaxation of the mapping selection problem and then carried out a detailed experimental evaluation of this approach.
Its many technical merits notwithstanding, the  work of \cite{Kimmig17} suffers from a serious drawback, namely, the objective function of the mapping selection problem is defined to be the sum of the size of the rules and the error (the number of the false positives and the false negatives). As stated earlier, however,  size  and  error are qualitatively different quantities, thus it is simply not meaningful to add them, the same way it is not meaningful to add dollars and miles if one is interested in a hotel room near the White House and is trying to minimize the cost of the room and the distance from the White House. This is why, to avoid this pitfall here, we first focus  on error minimization alone and then study the Pareto optimality of pairs of size and error values.

\looseness=-1 In the rule selection problem, the aim is to select a set of rules from a larger set of candidate rules based on some given data. There is a large body of work on the problem of deriving  a set of rules in data exchange and data integration from just  one or more given \emph{data examples}. There are several different approaches to this problem, including casting it as an optimization problem \cite{DBLP:journals/jacm/GottlobS10,DBLP:journals/tods/CateKQT17},
as a ``fitting" problem \cite{DBLP:journals/tods/AlexeCKT11,DBLP:conf/sigmod/AlexeCKT11},  as an interactive derivation problem \cite{Bonifati:2017:IMS:3035918.3064028},
or as a learning problem \cite{DBLP:journals/tods/CateDK13,DBLP:conf/pods/CateK0T18}. Clearly, this is a related but rather different problem because, in contrast to the rule selection problem, no candidate rules are part of the input.

\section{Basic Concepts and Algorithmic Problems} \label{sec:basic}

\noindent \textbf{Schemas and Instances}
A \emph{schema} \textbf{R} is a set $\{R_1,\ldots,R_k\}$ of relation symbols, each with a specified \emph{arity} indicating the number of its arguments.  An \emph{$\textbf{R}$-instance} $I$ is a set $\{R_1^I,\ldots,R^I_k\}$ of relations whose arities match those of the corresponding relation symbols.
An \emph{$\textbf{R}$-instance}  can be identified with the
set of all \emph{facts} $R_i(a_1,\ldots,a_m)$, such that $R_i$ is
a relation symbol in \textbf{R} and $(a_1,\ldots,a_m)$ is a tuple in
the relation $R^I_i$ of $I$ interpreting the relation symbol $R_i$.

\smallskip
\noindent \textbf{Rules}. Let \textbf{S} and \textbf{T} be two disjoint relational schemas.
In the rest of the paper, we will  refer to \textbf{S} as the {\em premise} schema and to  \textbf{T} as the {\em conclusion} schema. A \emph{rule} over \textbf{S} and \textbf{T} is a  Horn  formula of first-order logic of the form
\[
\forall\textbf{x}~ (\psi(\textbf{x}) \rightarrow P(\textbf{x})),
\]
where the premise $\psi(\textbf{x})$ is a conjunction of atoms over \textbf{S} and the conclusion $T(\textbf{x})$ is a single atom over \textbf{T} with variables among those in \textbf{x}.
 For example, the rule \[\forall x, y, z (E(x,z)\land E(z,y)\rightarrow F(x,y))\] asserts that $F$ contains all pairs of nodes connected via an $E$-path of length $2$.
 For simplicity, we will be dropping the universal quantifiers $\forall$, so that the preceding rule about paths of length $2$ will be written as  $(E(x,z)\land E(z,y)\rightarrow F(x,y))$.

\looseness = -1 The atoms in the premises  may contain constants or they may be built-in predicates, such as {\small {\bf jaccardSim}}. However,  none of the lower-bound complexity results established here uses such atoms, while the upper-bound complexity results hold true even in the presence of such atoms, provided the built-in predicates are polynomial-time computable.

   The {\em size} of a rule $\rho$, denoted by $|\rho|$,  is the number of atoms in the premise of $\rho$.  The \emph{size} of a collection $\mathcal{C}$ of rules, denoted by $|\mathcal{C}|$, is the sum of the sizes of the rules in $\mathcal{C}$.

\smallskip
\noindent \textbf{Data example.} A {\em data example} is a pair $(I,J)$, where $I$ is an instance over the premise schema \textbf{S} and $J$ is an instance over the conclusion schema \textbf{T}.

\smallskip
\noindent \textbf{Rule evaluation.}
Given a rule $\rho$ and an instance $I$,  we write \eval{\rho}{I} to denote the  result of evaluating the premise of $\rho$ on $I$ and then populating the conclusion of $\rho$ accordingly. For example, if $\rho$ is the rule $(E(x,z)\land E(z,y)\rightarrow F(x,y))$ and $E$ is a graph, then \eval{\rho}{E} is the set $F$ consisting of all pairs of nodes of $E$ connected via a path of length $2$.  If   $\mathcal{C}$ is a set of rules, then   \eval{\mathcal{C}}{I} is the set of facts $\bigcup_{\rho\in \mathcal{C}} \textup{Eval}(\rho,I)$.
In data exchange, computing \eval{\mathcal{C}}{I} amounts to running the \emph{chase} procedure \cite{FAGIN200589}.

In general, given a collection $\mathcal C$ of rules  and an instance $I$, computing \eval{\mathcal C}{I} is an exponential-time task; the source of the exponentiality is the maximum number of atoms  in the premises of the rules in $\mathcal C$ and the maximum arity of the relation symbols in the conclusions of the rules. If, however, both these quantities are bounded by  constants, then \eval{\mathcal C}{I} is computable in polynomial time, according to the following  fact (e.g., see \cite{FAGIN200589}).

\begin{proposition}
\label{fact:assumption}
Let  $a$ and $r$ be two fixed positive integers. Then the following problem is solvable in polynomial time: given an instance $I$ and a collection $\mathcal C$ of rules such that the maximum number of atoms in the premises of rules in $\mathcal C$ is at most $a$ and the maximum arity of the relation symbols in the conclusions of these rules is at most $r$, compute \eval{\mathcal C}{I}.
\end{proposition}

\noindent \textbf{False positive errors and false negative errors}.
Given  a collection $\mathcal{C}$ of rules and a data example $(I,J)$, a \emph{false positive error} is a fact $f$ in \eval{\mathcal{C}}{I} that is not in $J$, while a \emph{false negative error} is a fact $f$ in $J$ that is not in \eval{\mathcal{C}}{I}. We write $\textup{FP}(\mathcal{C},(I,J))$ and $\textup{FN}(\mathcal{C},(I,J))$ (or, simply, FP and FN) for the set of false positive and false negative errors with of $\mathcal{C}$ with respect to $(I,J)$, that is,
$$ \textup{FP} = \textup{Eval}(\mathcal{C},I) \setminus J ~~ \mbox{and} ~~
\textup{FN} = J \setminus \textup{Eval}(\mathcal{C},I).$$
%
We will focus on the following two optimization problems concerning the minimization of the number of  errors.

\begin{definition}\emph{
[\textsc{\optruleselectgeneral{FP+FN}}]
\label{def:ruleselect_fpfn}~\\
\underline{Input:} A set $\mathcal{C}$ of rules and a data example $(I,J)$.\\
\underline{Goal:}  Find a subset $\mathcal{C}^*\subseteq \mathcal{C}$ such that the sum of the number of the  false positive errors and  the number of  false negative errors 
 of $\mathcal{C}^*$ with respect to $(I,J)$,
 is minimized.}
\end{definition}
A \emph{feasible} solution of a given instance $\mathcal{C}$,  $(I,J)$ of \textsc{\optruleselect{FP+FN}} is  a subset $\mathcal{C'}$ of  $\mathcal{C}$. We write   $\error{FP+FN}{{\mathcal C'}}{(I,J)}$ to denote the \emph{error} of $\mathcal C'$ with respect to $(I,J)$, i.e.,  the sum of the number of the false positive errors and the number of false negative errors.

\begin{definition}\emph{
[\textsc{\optruleselectgeneral{FP}}]
\label{def:ruleselect_fp}~\\
\underline{Input:} A set $\mathcal{C}$ of rules and a data example $(I,J)$ such that  $J \subseteq$ \eval{\mathcal{C}}{I}.\\
\underline{Goal:}  Find a subset $\mathcal{C}^*\subseteq \mathcal{C}$ such that the number of false negative errors  is zero and the number of false positive errors of $\mathcal{C}^*$ with respect to $(I,J)$ is minimized.}
\end{definition}

A \emph{feasible} solution of a given instance $\mathcal{C}$,  $(I,J)$ of \textsc{\optruleselect{FP}} is a subset $\mathcal{C'}$ of $\mathcal{C}$ such that $J \subseteq$ \eval{\mathcal{C'}}{I}.  Feasible solutions always exist because $\mathcal{C}$ is one. We write   $\error{FP}{{\mathcal C'}}{(I,J)}$ to denote the \emph{error} of  $\mathcal C'$ with respect to $(I,J)$, i.e.,   the number of the false positive errors.

We do not consider \optruleselectgeneral{FN}, i.e., the optimization problem that aims to minimize the number of false negative errors. The reason is that there is a trivial solution to this problem, namely, we can select all the rules (if the number of false positive errors is not required to  be zero) or select all the rules that  produce no false positive errors (if the number of false positive errors is required to be zero).

To gauge the  difficulty of solving an optimization problem, one often studies two decision problems that underlie the optimization problem at hand: a decision problem about bounds on the optimum value and a decision problem about the exact optimum value. We introduce these two problems for each of the optimization problems in Definitions \ref{def:ruleselect_fpfn} and \ref{def:ruleselect_fp}.

\begin{definition}\emph{
\label{def:ruleselect_fpfn_decision_exact}
Given a set $\mathcal{C}$ of rules, a data example $(I,J)$,  and an integer $k$,
\begin{itemize}
\item \textsc{\ruleselectgeneral{FP+FN}} asks:
   is there a subset $\mathcal{C'}$ of $\mathcal C$  such that $\error{FP+FN}{\mathcal{C'}}{(I,J)}\leq k$?
\item \textsc{\exactruleselectgeneral{FP+FN}} asks:  is the optimum value of \textsc{\optruleselectgeneral{FP+FN}} on  $\mathcal{C}$ and $(I,J)$ equal to $k$?
\end{itemize}}
\end{definition}

\begin{definition}\emph{
\label{def:ruleselect_fp_decision_exact}
Given a set $\mathcal{C}$ of rules, a data example $(I,J)$ such that $J \subseteq$ \eval{\mathcal{C}}{I},  and an integer $k$,
\begin{itemize}
\item \textsc{\ruleselectgeneral{FP}} asks: is there a subset $\mathcal{C'}\subseteq \mathcal{C}$ such that the number of false negative errors of $\mathcal{C'}$ with respect to  $(I,J)$ is zero and
    $\error{FP}{\mathcal{C'}}{(I,J)}\leq k$?
\item \textsc{\exactruleselectgeneral{FP}} asks: is the optimum value of  \textsc{\optruleselectgeneral{FP}} on
$\mathcal{C}$ and $(I,J)$ equal to $k$?
\end{itemize}}
\end{definition}

\section{The Complexity of Error Minimization} \label{sec:error-min}

We will investigate the computational complexity  of the decision problems introduced in Definitions  \ref{def:ruleselect_fpfn_decision_exact} and \ref{def:ruleselect_fp_decision_exact} by considering parameterized versions of these problems with parameters the maximum number of atoms in the premises of the rules and the maximum arity of the relation symbols in the conclusion of the rules.

\begin{definition}\emph{
\label{def:ruleselect_fpfn_decision_exact param}
Let $a$ and $r$ be two fixed positive integers.
\begin{itemize}
\item \textsc{\ruleselect{FP+FN}} is the restriction of \textsc{\ruleselectgeneral{FP+FN}} to inputs in which the maximum number of atoms in the premises of rules in the given set $\mathcal C$ of rules is at most $a$ and the maximum arity of the relation symbols in the conclusions of these rules is at most $r$.
\item The decision problems \textsc{\ruleselect{FP}}, \textsc{\exactruleselect{FP+FN}},
\textsc{\exactruleselect{FP}}, \textsc{\optruleselect{FP+FN}}, and \textsc{\optruleselect{FP}} are defined in an analogous way.
\end{itemize}}
\end{definition}


\begin{table*}[]
\centering
\footnotesize
\begin{tabular}{|l|l|c| C{5.8cm}|}
\hline
\multicolumn{2}{| l |}{}	& \bf \small False Positive Errors  & \bf \small False Positive + False Negative Errors  \\
\specialrule{.2em}{.1em}{.1em}
\multicolumn{2}{|l|} \ruleselect
& NP-complete                      & NP-complete                         \\ \hline
\multicolumn{2}{|l|}\exactruleselect
 & DP-complete                      & DP-complete                         \\ \hline
\multicolumn{1}{|c|}{\multirow{2}{*}{\optruleselect}}	& approximation upper bound & $2\sqrt{|\mathcal{C}|\log{|J|}}$    &  $2\sqrt{(|\mathcal{C}|+|J|)\log{|J|}}$  \\ \cline{2-4}
\multicolumn{1}{|c|}{} & approximation  lower bound  &  $2^{\log^{1-\epsilon}({|\mathcal{C}|})}$, for every $\epsilon>0$ &
 $2^{\log^{1-\epsilon}({|J|})}$, for every $\epsilon > 0$
                              \\
\hhline{|=|=|=|=|}
\multicolumn{2}{|l|}{\textsc{Pareto Opt Solution}\textup(a,r) }
  & coNP-complete	 & coNP-complete                         \\ \hline
\multicolumn{2}{|l|}{\textsc{Pareto Front Membership}\textup(a,r)  }
 & DP-complete & DP-complete                         \\ \hline
 \multicolumn{2}{|l|}{\textsc{Bi-level Opt Solution}\textup(a,r)}
  & coNP-complete	 & coNP-complete                         \\ \hline
\multicolumn{2}{|l|}{\textsc{Bi-level Opt Value}\textup(a,r)}
 & DP-complete & DP-complete                         \\ \hline
\end{tabular}
\caption{Summary of Results ($\mathcal{C}$: set of input rules; $J$: input conclusion instance; approx.\ lower bounds assume that P$\not=$NP)}
\label{tbl:summary}
\end{table*}

\begin{theorem}
\label{thm:ruleselect_npc} Let $a$ and $r$ be two fixed positive integers.
The decision problems \textsc{\ruleselect{FP}}  and \textsc{\ruleselect{FP+FN}}  are {\rm NP}-complete.
\end{theorem}

\begin{proof}~\emph{(Sketch)}~
\eat{
The two problems are  in NP, because, given an instance $\mathcal C$, $(I,J)$, $k$,  we can guess a subset $\mathcal C'$ of $\mathcal C$ and then use Proposition \ref{fact:assumption} to compute \eval{{\mathcal C'}}{I} in polynomial time and, in each case, check whether the error is at most $k$.}
Membership in NP follows easily from Proposition  \ref{fact:assumption}.
The  NP-hardness of \textsc{\ruleselect{FP}}  and \textsc{\ruleselect{FP+FN}} is
 proved using a polynomial-time reduction from  \setcover, a problem which was among the twenty one NP-complete problems in Karp's seminal paper \cite{Kar72}. The same reduction is used for both \textsc{\ruleselect{FP}}  and \textsc{\ruleselect{FP+FN}}; however, the argument for the correctness of the reduction is different in each case. We give the argument for \textsc{\ruleselect{FP+FN}}.

The \setcover problem asks: given a finite set $U=\{u_1,\dots,u_m\}$, a collection $\mathcal{S}=\{S_1, \dots, S_p\}$ of subsets of $U$ whose union is $U$, and an integer $k$,
is there a cover of $U$ of size at most $k$? (i.e., a subset $\mathcal{S'}\subseteq \mathcal{S}$ such that the union of the members of $\mathcal {S'}$ is equal to $U$ and  $|\mathcal{S'}|\leq k$.)

Let $U=\{u_1,\dots,u_m\}$, $\mathcal{S}=\{S_1, \dots, S_p\}$, $k$  be an instance of \setcover.
For every $i$ with $1\leq i\leq p$, we introduce the rule $Set_i(x)\rightarrow B(x)$ and we let $\cal C$ be the set of all these rules. Thus, the premise schema consists of the unary relation symbols $Set_i$, $1\leq i\leq p$, and the conclusion schema consists of the unary relation symbol $B$.
We introduce $p$  new elements $a_1,\ldots, a_p$ and construct the premise instance
 $I$ with $Set^I_i=  S_i\cup \{a_i\}$, $1\leq i\leq p$  (intuitively, each new element $a_i$  encodes the index of the set $S_i$). We  also construct the conclusion instance $J$ with $B^J = U$.

\eat{
  Clearly, there is a 1-1 correspondence between covers $\mathcal {S'}$ of $U$ and subsets $\mathcal C'$ of $\mathcal C$ of the same size that have no negative errors with respect to $(I,J)$.  It follows that
 there is a cover of $U$ of size at most $k$ if and only if there is a subset $\mathcal C'$ of $\mathcal C$ that has no false negative errors  with respect to  $(I,J)$ and is such that
    $\error{FP}{\mathcal{C'}}{(I,J)}\leq k$. Consequently, \textsc{\ruleselect{FP}} is NP-hard.
}

\looseness=-1
We claim that
 there is a cover of $U$ of size at most $k$ if and only if there is a subset $\mathcal C'$ of $\mathcal C$  such that
    $\error{FP+FN}{\mathcal{C'}}{(I,J)}\leq k$.
     If there is a cover of $U$ of size at most $k$, then there is  a subset $\mathcal C'$ of $\mathcal C$ that has no false negative errors  with respect to  $(I,J)$ and is such that
    $\error{FP}{\mathcal{C'}}{(I,J)}\leq k$. Therefore, $\error{FP+FN}{\mathcal{C'}}{(I,J)} = \error{FP}{\mathcal{C'}}{(I,J)}\leq k$. Conversely, assume that there is  a subset $\mathcal C'$ of $\mathcal C$  such that
    $\error{FP+FN}{\mathcal{C'}}{(I,J)}\leq k$. If there are no false negative errors, then $\cal C'$ corresponds to a cover of $U$ of size at most $k$. So, assume  there are $t$ false negative errors, for some $t>0$. These must result from elements $u_{i_1}, \ldots, u_{i_t}$ of $U$ that are not in any of the sets $S_j$ that produce  rules $Set_j(x)\rightarrow B(x)$ in $\mathcal C'$. For each such element, there is a set  in $\mathcal S$ containing it; thus, the elements $u_{i_1}, \ldots, u_{i_t}$ can be covered using at most $t$ sets in $\mathcal S$. Let $\mathcal C''$ be the subset of $\mathcal C$ obtained by adding to $\mathcal C'$ the rules arising from these sets. Clearly, $\mathcal C''$ has no false negative errors, so it has $t$ fewer negative errors than $\mathcal C'$ does and at most $t$ more positive errors than $\mathcal C'$ does with respect to $(I,J)$. Thus, $\error{FP+FN}{\mathcal{C''}}{(I,J)}\leq \error{FP+FN}{\mathcal{C'}}{(I,J)}\leq  k$. It follows that the subset $\mathcal C''$ gives rise to a cover of $U$ of size at most $k$.
 This completes the proof that \textsc{\ruleselect{FP+FN}} is NP-hard.
\end{proof}

 Next, we consider  the problems \textsc{\exactruleselect{FP}}  and \textsc{\exactruleselect{FP+FN}}. The class DP consists of all decision problems that can be written as the conjunction of a problem in NP and a problem in coNP \cite{DBLP:conf/stoc/PapadimitriouY82}. The prototypical DP-complete problem is \textsc{SAT-UNSAT}: given two Boolean formulas $\varphi$ and $\psi$, is $\varphi$ satisfiable and $\psi$ unsatisfiable?
\eat{
Numerous other DP-complete problems have been identified, including some that
 involve \emph{critical} questions, such as: is a given Boolean $\varphi$ unsatisfiable, but if any one clause of $\varphi$ is removed, then the resulting formula is satisfiable?
 Other
    DP-complete problems involve \emph{exact} questions about optimization problems. In particular,}
    Furthermore,
    \textsc{Exact Clique} is DP-complete: given a graph $G$ and an integer $k$, is the size of the largest clique in $G$ exactly $k$?
Note that a DP-complete problem is both NP-hard and coNP-hard.

\begin{theorem}
\label{thm:exactruleselect_dpc}
Let $a$ and $r$ be two fixed positive integers.\\
The decision problems \textsc{\exactruleselect{FP}} and \textsc{\exactruleselect{FP+FN}}  are {\rm DP}-complete.
\end{theorem}
\begin{proof}
 \looseness = -1
 \emph{(Hint)} Using Proposition \ref{fact:assumption}, it is easy to see that both these problems are in the class DP. The problem \exactsetcover asks: given a set $U$, a collection $\mathcal S$ of subsets of $U$, and an integer $k$, is the size of smallest cover of $U$ exactly $k$?  Using the DP-hardness of \textsc{Exact Clique} and reductions in \cite{Kar72}, one can show that \exactsetcover is DP-hard. Finally, the reduction  in the proof of Theorem \ref{thm:ruleselect_npc} is also a  polynomial-time reduction of \exactsetcover to both
 \textsc{\exactruleselect{FP}} and \textsc{\exactruleselect{FP+FN}}. Consequently, these two  problems are DP-complete.
\end{proof}

\looseness=-1 In Theorems \ref{thm:ruleselect_npc} and \ref{thm:exactruleselect_dpc}, the restriction to sets of  rules with a bound $a$ on the number  of premise atoms and a bound $r$ on the  arity of relation symbols in the conclusion schema was used to obtain the complexity-theoretic upper bounds (membership in NP and, respectively, membership in DP). The matching lower bounds  hold true even if $a=1$ and $r=1$,  because
the rules used to prove NP-hardness in Theorem \ref{thm:ruleselect_npc} and DP-hardness in Theorem \ref{thm:exactruleselect_dpc} have the form $Set_i(x)\rightarrow B(x)$.

%

\looseness=-1 Observe also that, in the proofs of Theorems \ref{thm:ruleselect_npc} and \ref{thm:exactruleselect_dpc}, there is no fixed bound on the \emph{number} of relation symbols in the premise schema; instead, this number depends on the size of the given instance of \setcover and \exactsetcover.
  If we also impose a  fixed bound on the number of relation symbols,
  then the rule selection problems trivializes because, in this case, there is a fixed number of rules.
Our next result shows that the rule selection problems  become intractable if we fix the premise schema (hence, we also fix the number of relation symbols occurring in it), but impose no bounds on the number of atoms in the premises of the rules.

\begin{theorem} \label{thm:fixed-premise-schema}
Let $\bf S$ be a premise schema consisting of one unary and four binary relation symbols, and let $\bf T$ be a target schema consisting of a single unary relation symbol.

If the rules contain an arbitrary finite number of atoms in their premises,
then \textsc{\ruleselectgeneral{FP}} is {\rm NP}-hard and \textsc{\exactruleselectgeneral{FP}} is {\rm DP}-hard.  Similar  results hold for \textsc{\ruleselectgeneral{FP+FN}} and \textsc{\exactruleselectgeneral{FP+FN}}.
\end{theorem}
\begin{proof} {\em (Hint)} Assume that the premise schema $\bf S$ consists of the unary relation symbol $One$ and  binary relation symbols $S, Bit_0, Bit_1, Succ$; assume also that the conclusion schema consists of the unary relation symbol $B$. We simulate each rule
 $Set_i(x)\rightarrow B(x)$, $1\leq i \leq p$, via a rule $\sigma_{i,p}$ that, intuitively, asserts that ``if $z=i$ and $x\in Set_i$, then $x \in B$". Since $i\leq p$, we can write $i$ in binary notation using  $\lceil\log p\rceil $ bits. We encode this binary representation of $i$ via a conjunction of premise atoms involving the binary relation symbols $Bit_0, Bit_1, Succ$ and the unary relation symbol $One$. A typical rule $\sigma_{i,p}$ is of the form
 $$Bit_0(i_1,z)\land Bit_1(i_2,z) \land \cdots \land Bit_1(i_{\lceil\log p\rceil},z) \land One (i_1)  $$
 $$\land Succ(i_1,i_2)\land \cdots \land Succ(i_{\lceil \log p\rceil -1},i_{\lceil \log p \rceil})\land S(x,z) \rightarrow B(x),$$
 where the binary representation of $i$ using $\lceil\log p\rceil $ bits gives rise to the sequence of atoms $Bit_0$ and $Bit_1$ in the premise of the rule.
 If an instance of \setcover contains a set  $U$ and a collection ${\mathcal S}=\{S_1,\ldots,S_p\}$ of subsets of $U$, then we construct a premise instance $I$ as follows:
 \begin{itemize}
 \item  $S^I = \{(x,i): x \in S_i\cup \{a_i\}, 1\leq i \leq p\}$, where each $a_i$ is a new element;
 \item  $Bit_k^I= \{(j,i): i\leq p ~ \land ~ \mbox{the $j$-th bit of $i$ is $k$}\}$,  $k=0,1$;
\item $One^I = \{1\}$; $Succ^I = \{(j,j+1): 1\leq j\leq \lceil \log p\rceil -1\}$.
\end{itemize}
As  a conclusion instance, we construct $J$ with $B^J = U$.
\end{proof}
\section{The Approximation of Error Minimization} \label{sec:approximation}

The hardness results in Theorems \ref{thm:ruleselect_npc} and \ref{thm:exactruleselect_dpc} imply that, unless P=NP, there is no polynomial-time algorithm for solving exactly the optimization problems \optruleselect{FP} and \optruleselect{FP+FN}. As detailed in \cite{DBLP:books/fm/GareyJ79}, approximation algorithms offer a way to cope with the computational hardness of optimization problems. In the case of minimization problems, the goal is to design a polynomial-time algorithm that, given an instance of the optimization problem at hand,  returns a feasible solution such that the
the approximate value is less than or equal to   a certain factor of  the optimum value; in general, the factor depends on the size of the instance. For example, there is a polynomial-time algorithm that, given an instance $(U, \mathcal S)$ of the \textsc{Min Set Cover} problem, the algorithm returns a cover whose size is less than or equal to  $\ln({|\mathcal S|})$ times the size of the optimal cover of the instance $(U, \mathcal S)$.

 It is well known that different optimization problems may have widely different approximation properties;
for example,
\textsc{Knapsack} is $\epsilon$-approximable, for every $\epsilon > 0$, while  \textsc{Min Set Cover} is not $\epsilon$-approximable for any $\epsilon > 0$, unless P= NP
(see \cite{AroraB09} for an exposition).

Let $a$ and $r$ be two fixed positive integers. What are the approximation properties of \optruleselect{FP} and \optruleselect{FP+FN}? At first sight and in view of the reduction from \setcover in Theorems \ref{thm:ruleselect_npc} and \ref{thm:exactruleselect_dpc}, one may guess that they ought to be the same as those of \textsc{Min Set Cover}. As we shall see, though, the approximation properties of   \optruleselect{FP+FN} appear to be
 worse than those of \optruleselect{FP}, which, in turn, appear to be worse than those of \textsc{Min Set Cover}.

Note that polynomial time reductions between two decision problems need not preserve the approximation properties of the corresponding optimization problems. For example, \textsc{Set Cover} is polynomial-time reducible to \textsc{Node Cover}, yet \textsc{Min Node Cover} is constant-approximable, but \textsc{Min Set Cover} is not. There is, however, a  finer notion of polynomial-time reduction, called \emph{$\mathcal{L}$-reduction}, that preserves approximation properties between optimization problems; this notion, which was introduced in \cite{PAPADIMITRIOU1991}, is defined next.

 Let $Q$ and $Q'$ be two optimization problems. We say that $Q$ \emph{$\mathcal{L}$-reduces} to $Q'$ if there are two polynomial-time algorithms $f$, $g$, and two positive constants $\alpha$, $\beta$ such that for every instance $I$ of $Q$, the following hold:
 \begin{itemize}
 \item  Algorithm $f$ produces an instance $I’ =f(I)$ of $Q'$ such that the optimum values $\mbox{opt}_Q(I)$ of $Q$ on $I$ and   $\mbox{opt}_{Q'}(I')$ of  $Q'$ on $I$ and $I'$ satisfy
 $\mbox{opt}_Q(I) \leq \alpha \mbox{opt}_{Q'}(I')$.
 \item  For every solution $I'$ of $Q'$ with value $c'$,  algorithm $g$ produces a solution $I$ of $Q$ with value $c$ such that $|c - \mbox{opt}_Q(I)|\leq \beta
 |c' - \mbox{opt}_{Q'}(I')|$.
 \end{itemize}

   By Proposition 2 in \cite{PAPADIMITRIOU1991}, if $Q$ reduces to $Q'$ via an $\mathcal L$-reduction in which $\alpha = 1$ and $\beta =1$, then every polynomial-time approximation algorithm for $Q'$ with factor $\epsilon$ gives rise to a polynomial-time approximation algorithm for $Q$ with the same factor $\epsilon$.
We say that two optimization problems $Q$ and $Q'$ have \emph{the same approximation properties} if there are $\mathcal L$-reductions from $Q$ to $Q'$ and from $Q'$ to $Q$ in both of which $\alpha=1$ and $\beta=1$.

 We now bring into the picture the \redblue problem, a variant of \textsc{Min Set Cover} that will turn out to have the same approximation properties as \optruleselect{FP}. The  properties of  \redblue  were studied  in \cite{carr00}, where both upper and lower bounds for its approximability were obtained; improved upper bounds were  obtained in \cite{PELEG200755}.

\begin{definition}\rm{
[{\small\redblue problem}]~\\
\noindent \underline{Input:}
Two disjoint  sets $R=\{r_1,\dots, r_\alpha\}$ and $B=\{b_1,\dots, b_\beta\}$ of ``red" and ``blue" elements, and a collection $\mathcal{S}=\{S_1,\dots, S_m\}$ of subsets of $R\cup B$.\\
\looseness=-1 \underline{Goal:}  Find a subset $\mathcal{S}^*\subseteq \mathcal{S}$ that  covers all blue elements, but covers as few red elements as possible.}
\end{definition}

%


\begin{theorem}
\label{thm:ruleselectfp_approx}
Let $a$ and $r$ be two fixed positive integers.
\textsc{\optruleselect{FP}} and \redblue have the same approximation properties. Consequently, the following hold true for \textsc{\optruleselect{FP}}.
\begin{itemize}
    \item  \textsc{\optruleselect{FP}} is approximable  within a factor of $2\sqrt{|{\cal C}|\log{|J|}}$, where $|\cal C|$ is the number of input rules and $|J|$ is the size of the input conclusion instance $J$.
    \item Unless {\rm P=NP}, for every $\epsilon > 0$,  there is no polynomial time algorithm that approximates \textsc{\optruleselect{FP}}   within a factor of $2^{\log^{1-\epsilon}({|\mathcal{C}|})}$, where $|{\mathcal C}|$ is as above.
\end{itemize}
\end{theorem}
\begin{proof}\emph{(Sketch)}~
 We  show that \optruleselect{FP} reduces to \redblue via a $\mathcal L$-reduction with $\alpha=1$ and $\beta = 1$.   Given an instance $K=({\mathcal C}, (I,J))$ of \optruleselect{FP}, we construct  the following instance $K'=(R,B,{\mathcal S})$ of \redblue.

We put $B = J$
 and $R = \mbox{\eval{{\mathcal C}}{I}}\setminus J$.
 Thus, the blue elements are the facts of $J$, while the red elements are the facts of \eval{{\mathcal C}}{I} that are not in $J$.
 We form the collection   $\mathcal{S}$ consisting of all sets \eval{r}{I}, where $r$ is a rule in $\mathcal C$.
\eat{
By Proposition \ref{fact:assumption}, \eval{{\mathcal C}}{I} is polynomial-time computable in the size of $K$, so this is a polynomial-time reduction.}

\eat{
There is a 1-1 correspondence between feasible solutions for $K$ with $p$ positive errors and  feasible solutions for $K'$ that cover $p$ red elements. In particular, the optimum value of $K$ coincides with the optimum value of $K'$. Consequently,}
 \looseness = -1 It is easy to see that this is  a $\mathcal L$-reduction with $\alpha=1$ and $\beta = 1$.
 Thus, every approximation algorithm for \redblue gives rise to an approximation algorithm for \optruleselect{FP} with the same approximation factor.
In \cite{PELEG200755}, a polynomial-time algorithm with approximation factor of
$2\sqrt{|{\cal S}|\log{\beta}}$ for \redblue is described, where $\cal S$ is the given collection of sets and $\beta$ is the number of blue elements. This yields a polynomial-time algorithm for \optruleselect{FP} with an approximation factor of $2\sqrt{|{\cal C}|\log{|J|}}$.
 \eat{where $|\cal C|$ is the number of input rules and $|J|$ is the size of the input conclusion instance $J$.}
\eat{
Note that the size $|I|$ of the input premise structure does not appear in the approximation factor for \optruleselect{FP} because the number of red elements does not appear in the approximation factor for \redblue.}


Next, we show that \redblue reduces to \optruleselect{FP} via  a $\mathcal L$-reduction with $\alpha=1$ and $\beta = 1$.  Given an instance $K'=(R,B, {\mathcal S})$ of \redblue, we construct the following instance $K=({\mathcal C}, (I,J))$  of \optruleselect{FP}.

Let ${\mathcal S}=\{S_1,\ldots, S_m\}$. The premise schema  consists of unary relation symbols $Set_i$, $1\leq i\leq m$, and the conclusion schema consists of a unary relation symbol $R$.
    We put ${\mathcal C}=\{Set_i(x)\rightarrow R(x): 1\leq i\leq m\}$.
 We construct the premise instance $I$ by putting $Set^I_i=S_i$, $1\leq i\leq m$;  we construct the conclusion instance $J$ by putting $R^J = B$.

\eat{
There is a 1-1 correspondence between feasible solutions for $K'$ that cover $p$ red elements  and  feasible solutions for $K$ with $p$ false positive errors. In particular, the optimum value of $K'$ coincides with the optimum value of $K$. Consequently,}
 It is easy to see that this is a $\mathcal L$-reduction with $\alpha=1$ and $\beta = 1$. Thus,
 inapproximability results  for \redblue transfer to inapproximability results for  \optruleselect{FP}. In \cite{carr00}, it was shown that, unless P=NP,
for every $\epsilon > 0$,  there is no polynomial time algorithm that approximates \redblue    within a factor of $2^{\log^{1-\epsilon}({|\mathcal{S}|})}$. Thus, unless {\rm P=NP}, for every $\epsilon > 0$,  there is no polynomial time algorithm that approximates \textsc{\optruleselect{FP}}   within a factor of $2^{\log^{1-\epsilon}({|\mathcal{C}|})}$.
%
\end{proof}

 Finally, we bring into the picture the \pscproblem~(\psc)~problem, a  variant of \textsc{Min Set Cover} that will turn out to have the same approximation properties as \optruleselect{FP+FN}.
The  \psc~problem has been studied in \cite{MIETTINEN2008219}.

\begin{definition}\rm{
[{\small \pscproblem}]\\
\underline{Input:} Two disjoint sets $P$ and $N$ of ``positive" and ``negative" elements, and a collection $\mathcal{S}=\{S_1,\dots,S_m\}$ of subsets of $P\cup N$.\\
\underline{Goal:} Find a subset $\mathcal{S}^*\subseteq \mathcal{S}$ that minimizes the sum of the number of uncovered positive elements and the number of covered negative elements, i.e.,  the quantity
\[
\small
\text{cost}(\mathcal{S}^*) = |P\setminus (\cup \mathcal{S}^*)| + |N \cap (\cup \mathcal{S}^*)|.
\]}
\end{definition}

\begin{theorem}
\label{thm:ruleselectfpfn_approx}
Let $a$ and $r$ be two fixed positive integers.
\textsc{\optruleselect{FP+FN}} and \textsc{\psc} have the same approximation properties.
Consequently, the following hold true for \textsc{\optruleselect{FP+FN}}.
\begin{itemize}
    \item  \textsc{\optruleselect{FP+FN}} is approximable  within a factor of $2\sqrt{(|\mathcal{C}|+|J|)\log{|J|}}$, where $|\cal C|$ is the number of input rules and $|J|$ is the size of the input conclusion instance $J$.
    \item Unless {\rm P=NP}, for every $\epsilon > 0$,  there is no polynomial time algorithm that approximates \textsc{\optruleselect{FP+FN}}   within a factor of $2^{\log^{1-\epsilon}({|J|})}$, where $|J|$ is as above.
\end{itemize}
\end{theorem}
\begin{proof} \emph{(Hint)}~
 \looseness = -1
 The $\mathcal L$-reductions between  \optruleselect{FP+FN} and \psc are the same as those between \optruleselect{FP} and \redblue in the proof of Theorem \ref{thm:ruleselectfp_approx}, but with the positive elements playing the role of the red elements and with the negative elements that  of the negative elements.
\eat{
In \cite{MIETTINEN2008219}, it was shown that \psc has a polynomial-time algorithm with approximation factor of $2\sqrt{(|\mathcal{Q}|+|P|)\log{|P|}}$, where $|\mathcal{Q}|$ is the size of the collection of sets and $|P|$ is the number of positive elements.  It was also shown there that, unless P=NP, for every $\epsilon >0$, there is no polynomial time algorithm that approximates \psc   within a factor of $2^{\log^{1-\epsilon}({|P|})}$.}
 The upper and lower bounds for the approximability of \optruleselect{FP+FN} follow from such bounds for \psc.
\eat{
We first show that  \optruleselect{FP+FN} reduces to \psc via a $\mathcal L$-reduction with $\alpha=1$ and $\beta = 1$.  Given an instance $K=({\mathcal C}, (I,J))$ of \optruleselect{FP+FN}, we construct the following instance $K'=(P,N,{\mathcal C})$ of \psc.
\begin{itemize}
\item We put $P=J$ and $N=\mbox{\eval{{\mathcal C}}{I}}\setminus J$.
 \item We form the collection $\mathcal Q$ consisting of all sets $\mbox{\eval{r}{I}}$, where $r$ is a rule in $\mathcal C$.
\end{itemize}
Using a reasoning similar to the one in the proof of Theorem \ref{thm:ruleselectfp_approx}, we conclude that every approximation algorithm for \psc gives rise to an approximation algorithm for  \optruleselect{FP+FN} with the same approximation factor. In \cite{MIETTINEN2008219}, it was shown that \psc has a polynomial-time algorithm with approximation factor of $2\sqrt{(|\mathcal{Q}|+|P|)\log{|P|}}$. It was also shown that, unless P=NP, for every $\epsilon >0$,

To show \psc $\mathcal{L}$-reduces to \optruleselect{FP+FN}, we assume that we are given an instance $K$ of \psc consisting of
a set $N$ of negative elements,  a set $P$ of positive elements, and a collection $\mathcal{S}$ of subsets of $N\cup P$.
We construct an instance $K'$ of \optruleselect{FP+FN} as follows:
\squishlist
\item we let $\mathcal{C}$ be a set consisting of constraints $Set_i(x)\rightarrow Pos(x)$ for all $S_i\in\mathcal{S}$;
\item we let $I= Set^I_i\cup, \dots, \cup ~Set^I_\mathcal{S}$ be the premise instance, where $Set^I_i=\{Set_i(e) | e\in S_i\}$;
\item we let $J$ be the conclusion instance consisting of facts $Pos(p)$ for all $p\in P$.
\squishend
Then, we can show that there is a feasible solution for $K$ with cost $w$ iff there is a feasible solution for $K'$ with $w$ total errors. Moreover, given a feasible solution for $K$ with cost
$w$ we can compute a feasible solution for $K'$ with $w$ total errors in linear time, and vice versa.}
\end{proof}

%
%
%
%
%
%
%
%
%
%
%
%
%

\section{Bi-Objective and Bi-Level Minimization}
\label{section:pareto}
\looseness = -1 The optimization problems \optruleselect{FP} and \optruleselect{FP+FN}  have a single objective, namely,  the minimization of the error. Thus, all feasible solutions of minimum error are optimal, even though they may   differ in size.  What if one is also interested in taking the size of the solution into account? Since error and size are qualitatively incomparable quantities,  it is not meaningful to combine them into a single objective function by taking, say, a linear combination of the two. Instead, we can cast the problem as a bi-objective optimization problem and study the \emph{Pareto optimal solutions} that strike the right balance between error and size by capturing the trade-offs between these two quantities.

If $(a,b)$ and $(c,d)$ are two  pairs of integers, then $(a,b) \leq (c,d)$ denotes that $a\leq c$ and $b\leq d$, while $(a,b) < (c,d)$ denotes that $(a,b) \leq (c,d)$ and $(a,b) \not = (c,d)$.

\begin{definition}\rm{[Pareto Optimal Solution and Pareto Front]\\
Let $a$ and $r$ be two fixed positive integers and
let $K=(\mathcal{C}, (I,J))$ be an instance of \optruleselect{FP+FN}.
 \begin{itemize}
 \item  A \emph{Pareto optimal solution for $K$} is a subset $\mathcal {C^*}$ of $\mathcal C$ for which there is no subset $\mathcal C' \subseteq \mathcal C$  such that
      $(e',|{\mathcal C'}|) < (e^*, |{\mathcal {C^*}}|)$, where
      $e'= \error{FP+FN}{\mathcal C'}{(I,J)}$ and $e^* =  \error{FP+FN}{{\mathcal {C^*}}}{(I,J)}$.

  \item The \emph{Pareto front of $K$} is the set of all pairs  $(e^*,s^*)$ of integers such that there is a Pareto optimal solution $\mathcal {C^*}$ for $K$ with  $\error{FP+FN}{{\mathcal {C^*}}}{(I,J)})=e^*$ and
          $|{\mathcal C^*}|=s^*$.
     \end{itemize}
    }
     \end{definition}

%

The preceding notions of Pareto optimal solution and Pareto membership front give rise to natural decision problems. In what follows, we give the definitions for the parameterized versions of these problems.

\begin{definition}\rm{[Pareto Optimal Solution Problem]
\label{def:paretooptimality}\\
Let $a$ and $r$ be two fixed positive integers.

The \paretooptimality{FP+FN} problem asks:
given an instance $K=(\mathcal{C}, (I,J))$ of \optruleselect{FP+FN},
 and a subset  $\mathcal{C}^*$ of  $\mathcal{C}$,
is $\mathcal{C}^*$ a Pareto optimal solution  for $K$?}
\end{definition}

\begin{definition}\rm{[Pareto Front Membership Problem]
\label{def:paretomembership}~\\
Let $a$ and $r$ be two fixed positive integers.

The \paretofrontmembership{FP+FN} problem asks:  given an instance $K=(\mathcal{C}, (I,J))$ of \optruleselect{FP+FN} and  a pair $(e,s)$ of integers, is  $(e,s)$ on the Pareto front of $K$?}
\end{definition}

\eat{The notions of a Pareto optimal solution and Pareto front of an instance of \optruleselect{FP}, as well as}
The decision problems \paretooptimality{FP} and \paretofrontmembership{FP} are defined in an analogous manner.

\begin{theorem} \label{thm:pareto}
Let $a$ and $r$ be two fixed positive integers.
The following statements are true.
\begin{itemize}
\item The decision problems \textsc{\paretooptimality{FP}} and \textsc{\paretooptimality{FP+FN}} are {\rm coNP}-complete.
    \item  The decision problems \textsc{\paretofrontmembership{FP}} and \textsc{\paretofrontmembership{FP+FN}} are {\rm DP}-complete.
        \end{itemize}
        \end{theorem}
     \begin{proof} \emph{(Sketch)}
      \eat{
      The membership of the first two problems in coNP and the membership of the last two  in DP follow from the definition of each problem and Proposition   \ref{fact:assumption}.}
\looseness = -1 The coNP-hardness of the two Pareto optimality problems  is shown via  reductions from the coNP-complete problem \setcoveroptimality: given a set $U$, a collection ${\mathcal S}= \{S_1,\ldots,S_p\}$ of subsets of $U$ whose union is $U$, and a subset $\mathcal S'$ of $\mathcal S$, is $\mathcal S'$ an optimal cover of $U$?
The DP-hardness of the two Pareto front membership problems is shown via reductions from the \exactsetcover problem, encountered  in the proof of Theorem \ref{thm:exactruleselect_dpc}. In what follows,
we outline the reduction of \exactsetcover~to \paretofrontmembership{FP+FN}.

Let $U=\{u_1,\dots,u_m\}$, $\mathcal{S}=\{S_1, \dots, S_p\}$, $k$  be an instance of \exactsetcover.
For each $i$,  $1\leq i\leq p$, we introduce the rule $Set_i(x)\rightarrow B(x)$ and we let $\cal C$ be the set of all these rules.
%
We introduce $p$  new elements $a_1,\ldots, a_p$; furthermore, for every element $u_j\in U$, $1\leq j\leq m$, we introduce $p$ new elements $b_j^1,\ldots, b_j^p$, which can be thought of as ``clones" of $u_j$. We construct the premise instance
 $I$ with $Set^I_i=  S_i\cup \{a_i\}\cup \bigcup_{u_j\in S_i}\{b_j^1,\ldots, b_j^p\}$, $1\leq i\leq p$, that is, $Set^I_i$ consists of $a_i$, the elements of the set $S_i$, and all their clones.
    We  also construct the conclusion instance $J$ with $B^J = U\cup \bigcup_{u_j\in U} \{b_j^1,\ldots,b_j^p\}$, that is $B^J$ consists of the elements of $U$ and all their clones.
     Let $K= ({\mathcal C},  (I,J))$ be the resulting instance of \optruleselect{FP+FN}.

     We claim that the optimum size of the \setcover instance $(U,{\mathcal S})$ is $k$ if and only if the pair
     $(k,k)$ is on the Pareto front of the instance $K$. This claim follows by combining the following two facts. First, there is a 1-1 correspondence between covers $\mathcal {S'}$ of $U$ and subsets $\mathcal C'$ of $\mathcal C$ of the same size that have no negative errors with respect to $(I,J)$; moreover, for such subsets $\mathcal S'$, we have that $|{\mathcal S'}|= \error{FP+FN}{{\mathcal C'}}{(I,J)}$. Second, if a subset $\mathcal C'$ of $\mathcal C$ does not correspond to a cover of $U$, then
      $\error{FP+FN}{{\mathcal C'}}{(I,J)}\geq p+1> k$, since $k\leq p$ and,  in this case, there are at least $p+1$ false positives arising from an uncovered element of $U$ and its $p$-many clones.
\end{proof}

For bi-objective optimization problems, one would ideally like to efficiently produce the Pareto front of a given instance. For some such problems, this is impossible because the size (i.e., the number of points) of the Pareto front can be exponential in the size of a given instance. In the case of the Pareto front of rule selection problems, the size of the Pareto front is polynomial in the size of any given instance; nonetheless, Theorem \ref{thm:pareto} implies the following result.

\begin{corollary} \label{pareto:cor}
Let $a$ and $r$ be two fixed positive integers.
Unless \mbox{\rm{P=NP}}, there is no polynomial-time algorithm that, given an instance $K$ of
\textsc{\optruleselect{FP}} (or of \textsc{\optruleselect{FP+FN}}),  constructs the Pareto front of $K$.
\end{corollary}
\eat{
\begin{proof} If such an algorithm existed,  the decision problem  \paretofrontmembership{FP} (or  \paretofrontmembership{FP+FN}) would be solvable in polynomial time by constructing the Pareto front of the given instance $K$ and  checking whether or not the  given pair $(e,s)$ of integers is on the Pareto front.
\end{proof}
}

We conclude this section by identifying the complexity of the following two bi-level minimization problems  in which one minimizes first for error and then for size.

\begin{definition}\rm{[Bi-level Optimal Solution Problem]
\label{def:bileveloptimalsolution}\\
Let $a$ and $r$ be two fixed positive integers.

The \bileveloptimalsolution{FP+FN} problem asks:
given an instance $K=(\mathcal{C}, (I,J))$ of \optruleselect{FP+FN},
 and a subset  $\mathcal{C}^*$ of  $\mathcal{C}$,
is $\mathcal{C}^*$ a {\em bi-level optimal} solution? (i.e., is $\mathcal{C}^*$ an optimal solution of  \optruleselect{FP+FN}} that also has
minimal size among all such optimal solutions?)
\end{definition}

\begin{definition}\rm{[Bi-level Optimal Value Problem]
\label{def:bileveloptimalvalue}~\\
Let $a$ and $r$ be two fixed positive integers.

The \bileveloptimalvalue{FP+FN} problem asks:  given an instance $K=(\mathcal{C}, (I,J))$ of \optruleselect{FP+FN} and  a pair $(e,s)$ of integers,
is  $e$ the error of a bi-level optimal solution  and is $s$ its size?}
\end{definition}

\looseness=-1 The decision problems \bileveloptimalsolution{FP} and \bileveloptimalvalue{FP} are defined in an analogous manner. Note that if a pair $(e,s)$ is
a bi-level optimal value, then it is a point on the Pareto front. Moreover, $(e,s)$ is a special point because $e$ is the minimum possible error and $s$
is the minimum size of subsets of the given set of rules having this minimum error $e$.

The following result can be obtained by analyzing the proof of Theorem \ref{thm:pareto}.

\begin{corollary} \label{corollar:bilevel}
Let $a$ and $r$ be two fixed positive integers.
The following statements are true.
\begin{itemize}
\item The decision problems \textsc{\bileveloptimalsolution{FP}} and \textsc{\bileveloptimalsolution{FP+FN}} are {\rm coNP}-complete.
    \item  The decision problems \textsc{\bileveloptimalvalue{FP}} and \textsc{\bileveloptimalvalue{FP+FN}} are {\rm DP}-complete.
        \end{itemize}
        \end{corollary}

\section{Concluding Remarks}
\looseness = -1 We carried out a systematic complexity-theoretic investigation of  rule selection problems from several different angles, including an exploration of their complexity when they are cast as bi-objective optimization problems. A natural next step  is the implementation and experimental evaluation of the approximation algorithms for
\optruleselect{FP} and \optruleselect{FP+FN} based on corresponding approximation algorithms for \redblue and \pscproblem. Note that the approximation algorithms for the latter problems have been used
   for blocking function selection in entity resolution \cite{DBLP:conf/icdm/BilenkoKM06} and in view selection \cite{DBLP:conf/icdt/SarmaPGW10}. A different next step is to leverage the large literature on bi-objective optimization, including \cite{DBLP:conf/tacas/LegrielGCM10}, and design  heuristic algorithms for approximating the Pareto front of rule selection problems.

\paragraph{Acknowledgment} The research of Phokion Kolaitis was partially supported by NSF Grant IIS:1814152.
\bibliography{bib}

\begin{thebibliography}{}

\bibitem[\protect\citeauthoryear{Agrawal, Rantzau, and
  Terzi}{2006}]{DBLP:conf/sigmod/AgrawalRT06}
Agrawal, R.; Rantzau, R.; and Terzi, E.
\newblock 2006.
\newblock Context-sensitive ranking.
\newblock In Chaudhuri, S.; Hristidis, V.; and Polyzotis, N., eds., {\em
  Proceedings of the {ACM} {SIGMOD} International Conference on Management of
  Data, Chicago, Illinois, USA, June 27-29, 2006},  383--394.
\newblock {ACM}.

\bibitem[\protect\citeauthoryear{Alexe \bgroup et al\mbox.\egroup
  }{2011a}]{DBLP:journals/tods/AlexeCKT11}
Alexe, B.; {\SortNoop{Cate}}ten~Cate, B.; Kolaitis, P.~G.; and Tan, W.~C.
\newblock 2011a.
\newblock Characterizing schema mappings via data examples.
\newblock {\em {ACM} Trans. Database Syst.} 36(4):23:1--23:48.

\bibitem[\protect\citeauthoryear{Alexe \bgroup et al\mbox.\egroup
  }{2011b}]{DBLP:conf/sigmod/AlexeCKT11}
Alexe, B.; {\SortNoop{Cate}}ten~Cate, B.; Kolaitis, P.~G.; and Tan, W.~C.
\newblock 2011b.
\newblock Designing and refining schema mappings via data examples.
\newblock In Sellis, T.~K.; Miller, R.~J.; Kementsietsidis, A.; and Velegrakis,
  Y., eds., {\em Proceedings of the {ACM} {SIGMOD} International Conference on
  Management of Data, {SIGMOD} 2011, Athens, Greece, June 12-16, 2011},
  133--144.
\newblock {ACM}.

\bibitem[\protect\citeauthoryear{Arasu, Re, and Suciu}{2009}]{Dedupalog09}
Arasu, A.; Re, C.; and Suciu, D.
\newblock 2009.
\newblock {Large-Scale Deduplication with Constraints using Dedupalog}.
\newblock In {\em ICDE},  952--963.

\bibitem[\protect\citeauthoryear{Arora and Barak}{2009}]{AroraB09}
Arora, S., and Barak, B.
\newblock 2009.
\newblock {\em Computational Complexity - {A} Modern Approach}.
\newblock Cambridge University Press.

\bibitem[\protect\citeauthoryear{Bach \bgroup et al\mbox.\egroup
  }{2017}]{DBLP:journals/jmlr/BachBHG17}
Bach, S.~H.; Broecheler, M.; Huang, B.; and Getoor, L.
\newblock 2017.
\newblock Hinge-loss markov random fields and probabilistic soft logic.
\newblock {\em Journal of Machine Learning Research} 18:109:1--109:67.

\bibitem[\protect\citeauthoryear{Bilenko, Kamath, and
  Mooney}{2006}]{DBLP:conf/icdm/BilenkoKM06}
Bilenko, M.; Kamath, B.; and Mooney, R.~J.
\newblock 2006.
\newblock Adaptive blocking: Learning to scale up record linkage.
\newblock In {\em Proceedings of the 6th {IEEE} International Conference on
  Data Mining {(ICDM} 2006), 18-22 December 2006, Hong Kong, China},  87--96.
\newblock {IEEE} Computer Society.

\bibitem[\protect\citeauthoryear{Bonifati \bgroup et al\mbox.\egroup
  }{2017}]{Bonifati:2017:IMS:3035918.3064028}
Bonifati, A.; Comignani, U.; Coquery, E.; and Thion, R.
\newblock 2017.
\newblock Interactive mapping specification with exemplar tuples.
\newblock In {\em Proceedings of the 2017 ACM International Conference on
  Management of Data}, SIGMOD '17,  667--682.
\newblock New York, NY, USA: ACM.

\bibitem[\protect\citeauthoryear{Burdick \bgroup et al\mbox.\egroup
  }{2016}]{BurdickFKPT16}
Burdick, D.; Fagin, R.; Kolaitis, P.~G.; Popa, L.; and Tan, W.
\newblock 2016.
\newblock {A Declarative Framework for Linking Entities}.
\newblock {\em {ACM} Trans. Database Syst.} 41(3):17.
\newblock Preliminary version appeared in ICDT, pages 25--43, 2015.

\bibitem[\protect\citeauthoryear{Carr \bgroup et al\mbox.\egroup
  }{2000}]{carr00}
Carr, R.~D.; Doddi, S.; Konjevod, G.; and Marathe, M.
\newblock 2000.
\newblock On the red-blue set cover problem.
\newblock In {\em Proceedings of the Eleventh Annual ACM-SIAM Symposium on
  Discrete Algorithms}, SODA '00,  345--353.
\newblock Philadelphia, PA, USA: Society for Industrial and Applied
  Mathematics.

\bibitem[\protect\citeauthoryear{{\SortNoop{Cate}}ten~Cate \bgroup et
  al\mbox.\egroup }{2017}]{DBLP:journals/tods/CateKQT17}
{\SortNoop{Cate}}ten~Cate, B.; Kolaitis, P.~G.; Qian, K.; and Tan, W.
\newblock 2017.
\newblock Approximation algorithms for schema-mapping discovery from data
  examples.
\newblock {\em {ACM} Trans. Database Syst.} 42(2):12:1--12:41.

\bibitem[\protect\citeauthoryear{{\SortNoop{Cate}}ten~Cate \bgroup et
  al\mbox.\egroup }{2018}]{DBLP:conf/pods/CateK0T18}
{\SortNoop{Cate}}ten~Cate, B.; Kolaitis, P.~G.; Qian, K.; and Tan, W.
\newblock 2018.
\newblock Active learning of {GAV} schema mappings.
\newblock In den Bussche, J.~V., and Arenas, M., eds., {\em Proceedings of the
  37th {ACM} {SIGMOD-SIGACT-SIGAI} Symposium on Principles of Database Systems,
  Houston, TX, USA, June 10-15, 2018},  355--368.
\newblock {ACM}.

\bibitem[\protect\citeauthoryear{{\SortNoop{Cate}}ten~Cate, Dalmau, and
  Kolaitis}{2013}]{DBLP:journals/tods/CateDK13}
{\SortNoop{Cate}}ten~Cate, B.; Dalmau, V.; and Kolaitis, P.~G.
\newblock 2013.
\newblock Learning schema mappings.
\newblock {\em {ACM} Trans. Database Syst.} 38(4):28:1--28:31.

\bibitem[\protect\citeauthoryear{Davis, Schwarz, and
  Terzi}{2009}]{DBLP:conf/sdm/DavisST09}
Davis, W.~L.; Schwarz, P.~M.; and Terzi, E.
\newblock 2009.
\newblock Finding representative association rules from large rule collections.
\newblock In {\em Proceedings of the {SIAM} International Conference on Data
  Mining, {SDM} 2009, April 30 - May 2, 2009, Sparks, Nevada, {USA}},
  521--532.
\newblock {SIAM}.

\bibitem[\protect\citeauthoryear{de Amo \bgroup et al\mbox.\egroup
  }{2015}]{DBLP:journals/is/AmoDDGLS15}
de~Amo, S.; Diallo, M.~S.; Diop, C.~T.; Giacometti, A.; Li, D.~H.; and Soulet,
  A.
\newblock 2015.
\newblock Contextual preference mining for user profile construction.
\newblock {\em Inf. Syst.} 49:182--199.

\bibitem[\protect\citeauthoryear{Fagin \bgroup et al\mbox.\egroup
  }{2005}]{FAGIN200589}
Fagin, R.; Kolaitis, P.~G.; Miller, R.~J.; and Popa, L.
\newblock 2005.
\newblock Data exchange: semantics and query answering.
\newblock {\em Theoretical Computer Science} 336(1):89 -- 124.
\newblock Database Theory.

\bibitem[\protect\citeauthoryear{Garey and
  Johnson}{1979}]{DBLP:books/fm/GareyJ79}
Garey, M.~R., and Johnson, D.~S.
\newblock 1979.
\newblock {\em Computers and Intractability: {A} Guide to the Theory of
  NP-Completeness}.
\newblock W. H. Freeman.

\bibitem[\protect\citeauthoryear{Gottlob and
  Senellart}{2010}]{DBLP:journals/jacm/GottlobS10}
Gottlob, G., and Senellart, P.
\newblock 2010.
\newblock Schema mapping discovery from data instances.
\newblock {\em J. {ACM}} 57(2):6:1--6:37.

\bibitem[\protect\citeauthoryear{Karp}{1972}]{Kar72}
Karp, R.
\newblock 1972.
\newblock Reducibility among combinatorial problems.
\newblock In Miller, R., and Thatcher, J., eds., {\em Complexity of Computer
  Computations}. Plenum Press.
\newblock  85--103.

\bibitem[\protect\citeauthoryear{Kimmig \bgroup et al\mbox.\egroup
  }{2017}]{Kimmig17}
Kimmig, A.; Memory, A.; Miller, R.~J.; and Getoor, L.
\newblock 2017.
\newblock A collective, probabilistic approach to schema mapping.
\newblock In {\em 2017 IEEE 33rd International Conference on Data Engineering
  (ICDE)},  921--932.

\bibitem[\protect\citeauthoryear{Legriel \bgroup et al\mbox.\egroup
  }{2010}]{DBLP:conf/tacas/LegrielGCM10}
Legriel, J.; Guernic, C.~L.; Cotton, S.; and Maler, O.
\newblock 2010.
\newblock Approximating the pareto front of multi-criteria optimization
  problems.
\newblock In Esparza, J., and Majumdar, R., eds., {\em Tools and Algorithms for
  the Construction and Analysis of Systems, 16th International Conference,
  {TACAS} 2010, Proceedings}, volume 6015 of {\em Lecture Notes in Computer
  Science},  69--83.
\newblock Springer.

\bibitem[\protect\citeauthoryear{Lenzerini}{2002}]{DBLP:conf/pods/Lenzerini02}
Lenzerini, M.
\newblock 2002.
\newblock Data integration: {A} theoretical perspective.
\newblock In Popa, L.; Abiteboul, S.; and Kolaitis, P.~G., eds., {\em
  Proceedings of the Twenty-first {ACM} {SIGACT-SIGMOD-SIGART} Symposium on
  Principles of Database Systems, June 3-5, Madison, Wisconsin, {USA}},
  233--246.
\newblock {ACM}.

\bibitem[\protect\citeauthoryear{Miettinen}{2008}]{MIETTINEN2008219}
Miettinen, P.
\newblock 2008.
\newblock On the positive-negative partial set cover problem.
\newblock {\em Information Processing Letters} 108(4):219 -- 221.

\bibitem[\protect\citeauthoryear{Papadimitriou and
  Yannakakis}{1982}]{DBLP:conf/stoc/PapadimitriouY82}
Papadimitriou, C.~H., and Yannakakis, M.
\newblock 1982.
\newblock The complexity of facets (and some facets of complexity).
\newblock In Lewis, H.~R.; Simons, B.~B.; Burkhard, W.~A.; and Landweber,
  L.~H., eds., {\em Proceedings of the 14th Annual {ACM} Symposium on Theory of
  Computing, May 5-7, 1982, San Francisco, California, {USA}},  255--260.
\newblock {ACM}.

\bibitem[\protect\citeauthoryear{Papadimitriou and
  Yannakakis}{1991}]{PAPADIMITRIOU1991}
Papadimitriou, C.~H., and Yannakakis, M.
\newblock 1991.
\newblock Optimization, approximation, and complexity classes.
\newblock {\em Journal of Computer and System Sciences} 43(3):425 -- 440.

\bibitem[\protect\citeauthoryear{Peleg}{2007}]{PELEG200755}
Peleg, D.
\newblock 2007.
\newblock Approximation algorithms for the label-covermax and red-blue set
  cover problems.
\newblock {\em Journal of Discrete Algorithms} 5(1):55 -- 64.

\bibitem[\protect\citeauthoryear{Qian, Popa, and Sen}{2017}]{Qian2017}
Qian, K.; Popa, L.; and Sen, P.
\newblock 2017.
\newblock Active learning for large-scale entity resolution.
\newblock In {\em Proceedings of the 2017 ACM on Conference on Information and
  Knowledge Management}, CIKM '17,  1379--1388.
\newblock New York, NY, USA: ACM.

\bibitem[\protect\citeauthoryear{Sarma \bgroup et al\mbox.\egroup
  }{2010}]{DBLP:conf/icdt/SarmaPGW10}
Sarma, A.~D.; Parameswaran, A.~G.; Garcia{-}Molina, H.; and Widom, J.
\newblock 2010.
\newblock Synthesizing view definitions from data.
\newblock In Segoufin, L., ed., {\em Database Theory - {ICDT} 2010, 13th
  International Conference, Proceedings}, {ACM} International Conference
  Proceeding Series,  89--103.
\newblock {ACM}.

\bibitem[\protect\citeauthoryear{Singla and Domingos}{2006}]{SinglaDomingos06}
Singla, P., and Domingos, P.
\newblock 2006.
\newblock {Entity Resolution with Markov Logic}.
\newblock In {\em ICDM},  572--582.

\end{thebibliography}
\bibliographystyle{aaai}
\end{document}